%% file: neurips_2025.tex
\documentclass{article}
\usepackage[dblblindworkshop, final]{neurips_2025}
\usepackage[utf8]{inputenc} 
\usepackage[T1]{fontenc}    
\usepackage{hyperref}       
\usepackage{url}            
\usepackage{booktabs}       
\usepackage{amsfonts}       
\usepackage{nicefrac}       
\usepackage{microtype}      
\usepackage{xcolor}         
\usepackage{amssymb}
\usepackage{amsfonts}
\usepackage{amsthm}
\usepackage{mathtools}
\usepackage{float}
\usepackage{tikz-cd}
\usepackage{cancel}
\usepackage{booktabs}
\usepackage{panos-NEURIPS25}
\usepackage[textsize=tiny]{todonotes}
\usepackage{tcolorbox}
\usepackage{dirtytalk}
\newtheorem{theorem}{Theorem}
\newtheorem*{theorem*}{Theorem}
\newtheorem{lemma}[theorem]{Lemma}
\newtheorem{lemma*}{Lemma}
\newtheorem{proposition}[theorem]{Proposition}
\newtheorem{proposition*}{Proposition}
\newtheorem{corollary}[theorem]{Corollary}

\newtheorem{assumption}{Assumption}
\theoremstyle{remark}

\AddToHook{cmd/appendix/before}{%
  \setcounter{equation}{0}%
  \setcounter{theorem}{0}%
}

\title{An Eulerian Perspective on Straight-Line Sampling}
\workshoptitle{Frontiers in Probabilistic Inference: Sampling Meets Learning}

\author{%
  Panos Tsimpos \\
  Operations Research Center \\
  MIT, Cambridge, MA 02139, USA \\
  \texttt{ptsimpos@mit.edu} \\
  \And
  Youssef Marzouk \\
  Laboratory of Information and Decision Systems \\
  MIT, Cambridge, MA 02139, USA \\
  \texttt{ymarz@mit.edu} \\
}

\begin{document}

\maketitle
\begin{abstract}
  We study dynamic measure transport for generative modeling: specifically, flows induced by stochastic processes that bridge a specified source and target distribution. The conditional expectation of the process' velocity defines an ODE whose flow map achieves the desired transport. We ask \emph{which processes produce straight-line flows}---i.e., flows whose pointwise acceleration vanishes and thus are exactly integrable with a first-order method? We provide a concise PDE characterization of straightness as a balance between conditional acceleration and the divergence of a weighted covariance (Reynolds) tensor. Using this lens, we fully characterize affine-in-time interpolants 
  and show that straightness occurs exactly under deterministic endpoint couplings. We also derive necessary conditions that constrain flow geometry 
  for general processes, offering broad guidance for designing transports that are easier to integrate.
\end{abstract}

\section{Introduction}
Sampling from complex probability distributions is central to probabilistic inference and modern generative modeling. A recent line of work establishes \emph{dynamic measure transport} as a unifying paradigm: construct a stochastic process \((X_t)_{t\in[0,1]}\) whose marginals interpolate from a tractable source distribution \(\mu_0\) to a target distribution  \(\mu_1\), estimate the \emph{conditional velocity} 
$
v_t(x) \coloneq \mathbb{E}[\dot X_t \mid X_t=x]
$
and generate samples by evaluating the ODE flow maps \(\phi_t\) defined by
$
\partial_{t} \phi_t(x) = v_t\!\left(\phi_t(x)\right)$ with initial condition  $\phi_0(x)=x$. 
This perspective underlies methods such as \emph{stochastic interpolants}~\cite{albergo2023stochastic,albergo2022building}, \emph{flow matching}~\cite{lipman2022flow,tong2023cfm}, and \emph{score-based probability flow ODEs}~\cite{song2021scorebased}, as well as \emph{rectified flows}~\cite{liu2022flow,liu2022marginalpreserving,bansal2025wasserstein,hertrich2025relation}---all of which have demonstrated strong empirical performance.

The computational efficiency of these methods hinges on the \emph{geometry of the induced flow}. Generic flows demand many velocity-oracle evaluations because numerical integration error scales with the curvature (and higher derivatives) of \(\phi_t\). In contrast, if the flow is \emph{straight}, meaning that the acceleration of the flow map vanishes,
\[
\partial_{tt} \, \phi_t(x) \equiv 0 \quad \text{for all } (x,t),
\]
then \(\phi_t(x)\) is affine in \(t\): \(\phi_t(x)=(1-t)x+t\,\phi_1(x)\). Consequently, any first-order integrator is exact; one can traverse the entire path with \emph{a single} velocity evaluation. This motivates a fundamental question:
\begin{tcolorbox}[colback=white,colframe=black,left=1mm,right=1mm,
  top=1mm,bottom=1mm]
\emph{Which stochastic processes \((X_t)_t\) with \(X_0\!\sim\!\mu_0\) and \(X_1\!\sim\!\mu_1\) induce straight flows?}
\end{tcolorbox}
Prior work (e.g., \cite{liu2022flow,liu2022marginalpreserving,bansal2025wasserstein}) offers algorithmic frameworks that convert a given non-straight flow into a straight one, but a \emph{structural characterization} of when straightness is intrinsic to the underlying process has remained open. This paper develops such a theory. 

\paragraph{Contributions.} Our key contributions are as follows:
\begin{enumerate}
    \item \emph{PDE criterion for straightness.} We derive a new balance law, equation~\eqref{eq:reformulation-2}, that characterizes straight flows.
    \item \emph{Complete analysis of linear interpolations.} For processes of the form $X_t = (1-t) \, X_0 + t \, X_1$, with $X_0 \sim \mu_0$ and $X_1 \sim \mu_1$, 
    we show that the resulting flow is straight if and only if $(X_t)_t$ is constructed from a deterministic coupling of $\mu_0$ and $\mu_1$.
    \item \emph{Necessary conditions for the general case.} We obtain geometric constraints that any straight line-inducing \((X_t)_{t \in [0,1]}\) must satisfy.
\end{enumerate}

\paragraph{Scope and implications.} This short paper is fully \emph{theoretical} and applies to a broad class of stochastic processes. It offers a \emph{new PDE lens} on flow models and their straightness. We expect this framework to {aid the principled design} of stochastic processes for {sampling}; here we extract only a few immediate consequences and highlight open directions for theoretical and algorithmic follow-up. 

\section{Main results}
\subsection{Preliminaries}
For the standard notation used in this paper, please refer to Section~\ref{subsec:notation}.
Below, we review less typical notation used in our work. Fix a stochastic process $X \coloneq (X_t)_{t\in [0,1]}$ with sample paths in $W^{2,2}([0,1]; \mathbb{R}^d)$. 
We define the \emph{conditional velocity} and \emph{conditional acceleration} fields, also referred to as the \emph{ensemble velocity} and \emph{ensemble acceleration}, by
$
  v_t(x) \coloneq \mathbb{E} \left[ \, \dot X_t \mid X_t = x \right]$ and $a_t(x) \coloneq \mathbb{E} \left[ \, \ddot X_t \mid X_t = x \right]
  $.
We define the \emph{second moment velocity tensor} and the \emph{covariance} or \emph{Reynolds stress tensor} as
$
  \Sigma_t(x) \coloneq \mathbb{E} \left[ \, \dot X_t \otimes \dot X_t \mid X_t = x \right]$
  and
$
  \Pi_t(x) \coloneq \Sigma_t(x) - v_t(x) \otimes v_t(x)
$, respectively.
We let $\mu_t = \textup{Law}(X_t)$ be the marginal law of the process $X$ at time $t$ and write $\rho_t$ for the 
density of $\mu_t$ with respect to the Lebesgue measure, if it exists.
Furthermore, given a velocity field $v \coloneq (v_t)_t$, where $t \in [0,1]$ and $v_t: D \subset \Rd \to \Rd$, we define the induced \emph{flow} $\phi_t: \Rd \to \Rd$ to be the solution map to the ODE 
$
    \partial_{t} \, \phi_t(x) = v_t\!\left(\phi_t(x)\right)$ with initial condition $\phi_0(x)=x
$. 
We call a flow \emph{straight} if it is of the form 
$
    \phi_t(x) = (1-t) \, x + t \, \phi_1(x)
$, 
which is equivalent to $\partial_{tt} \, \phi_t(x) = 0$.
We define the \emph{material derivative} of $v_t$ at $x \in \Rd$ by 
$
    D_t \, v_t(x) = \partial_t \, v_t(x) + \left(v_t(x) \cdot \nabla\right) \, v_t(x)
$. 
Finally, for matrices $A, B \in \R^{d_1 \times d_2}$ we denote the \emph{Frobenius inner product} by 
$
  A : B \coloneq \textup{Tr}(A^\top B) = \sum_{i=1}^{d_1} \sum_{j=1}^{d_2} A_{ij} B_{ij},
$ 
and for a matrix field $T: D \subset \Rd \to \R^{d \times d}$ consisting of differentiable entries we define its \emph{divergence} as the vector field $\nabla \cdot T: D \to \Rd$ with components
$
  (\nabla \cdot T)_j (x) = \sum_{i=1}^d \partial_i T_{ij}(x) 
$.

\subsection{PDE characterization of straight-line flows.}
Fix $D \subset \Rd$ compact. We are concerned with the following problem:
\begin{tcolorbox}[colback=white,colframe=black,left=1mm,right=1mm,
  top=1mm,bottom=1mm]
    \textbf{Problem:} Given $\mu_0, \mu_1 \in \mathcal{P}(D)$, characterize all stochastic processes $(X_t)_{t\in [0,1]}$, with $X_i \sim \mu_i$ for $i \in \{0,1\}$, such that the flow $(\phi_t)_t$ generated by the ensemble velocity $(v_t)_t$ satisfies:
    \begin{equation*}
        \partial_{tt} \, \phi_t(x) = 0 \, , \quad \forall \, t \in [0,1] \, .
    \end{equation*}
\end{tcolorbox}
Before we start, we make some regularity assumptions that simplify the analysis:
\begin{assumption}
    \label{assumptoion:regularity}
    The marginal densities $\rho_t$ of $X_t$ exist and are positive
    ($\rho_t > 0$) on the compact set $D \subset \Rd$, and vanish on $D^c$.
    Also, sample paths of $X_t$ are in $W^{2,2}([0,1];\Rd)$ and the induced flow maps $t \mapsto \phi_t(x)$ are in $C^2([0,1];\Rd)$ for each $x \in D$.
\end{assumption}
First, an elementary reformulation allows us to move from the flow $\phi_t$ to the ensemble velocity.
\begin{proposition}\label{prop:reformulaton-1}
   For all $(t, x) \in [0,1] \times D$, one has
   \begin{equation*}
       \partial_{tt} \, \phi_t(x) = 0 \quad \iff \quad D_t v_t(x) = 0 \, .
   \end{equation*}
\end{proposition}
Please see Appendix~\ref{appendix-A} for the proof.
This result essentially constitutes a passage from the Lagrangian to the Eulerian perspective.
Rather than tracking the motion of a single particle via the flow map, the Eulerian perspective considers the \emph{total} change of the ensemble velocity field.

The second step is to relate the material derivative to other statistical quantities of interest, in particular the second order tensors $\Sigma_t$ and $\Pi_t$. Here, a \emph{momentum balance} identity allows
us to make progress.
\begin{lemma}
  \label{eq:second-order-continuity}
  The following equation holds for all $t \in [0,1]$ and $x \in D \subseteq \Rd$:
  \begin{equation} \label{eq:momentum}
    \partial_{t} \,(\rho_t \, v_t) + \nabla \cdot (\rho_t \, \Sigma_t) = \rho_t \, a_t \, .
  \end{equation}
\end{lemma}
For the proof, refer to Appendix~\ref{appendix-A}.
This is a fundamental identity with an elegant physical interpretation. Fix a control volume $dV$. From left to right in \eqref{eq:momentum}: the first term is the rate of change of the momentum in $dV$, the second term is the net momentum flux out of $dV$, and the third is the net body force acting on $dV$. 
Put differently, this is a manifestation of the $F = m a$ and $\dot p_t = F_t$ relations, where $p$ is momentum, $a$ is acceleration, $m$ is mass, and $F$ is force.

Finally, inserting the definition of the Reynolds tensor $\Pi_t$ together with the continuity equation into the above identity, we obtain yet another identity, elucidating the connection between the material derivative and the Reynolds (i.e., covariance) tensor.
\begin{lemma}
  The following equation holds for all $(x,t) \in [0,1] \times D$:
  \begin{equation*}
    \rho_t \, D_t v_t + \nabla \cdot \left( \rho_t \, \Pi_t \right) = \rho_t \, a_t \, .
  \end{equation*}
\end{lemma}
This immediately leads to a corollary and a reformulation of our problem.
\begin{corollary}
  \label{corollary:straightness}
  For all $(t,x) \in [0,1] \times D$ we have:
  \begin{equation*}
    D_t v_t = 0 \, \iff \, \nabla \cdot \left( \rho_t \, \Pi_t \right) = \rho_t \, a_t \, .
  \end{equation*}
\end{corollary}
\begin{tcolorbox}[colback=white,colframe=black,left=1mm
  ,right=1mm,
  top=1mm,bottom=1mm]
\textbf{Problem reformulation:} Given $\mu_0, \mu_1 \in \mathcal{P}(D)$, characterize all stochastic processes $(X_t)_{t\in [0,1]}$, with $X_i \sim \mu_i$ for $i \in \{0,1\}$, such that
\begin{equation}
  \label{eq:reformulation-2}
  \nabla \cdot \left( \rho_t \, \Pi_t \right) = \rho_t \, a_t \, , \quad \forall \, t \in [0,1] \, .
\end{equation}
\end{tcolorbox}

\subsection{Linear characteristics}
\label{subsec:linear_characteristics}
In this section we demonstrate the usefulness of reformulation~\eqref{eq:reformulation-2} to obtain a complete characterization of affine processes that give rise to straight flows. 
We define an \emph{affine process} with marginals $\mu_0, \mu_1 \in \mathcal{P}(\Rd)$ to be a process of the form
\begin{equation*}
    X_t = (1-t) \, X + t \, Y \, ,
\end{equation*}
where $(X, Y) \sim \gamma$ and $\gamma$ is a coupling of the measures $\mu_0$ and $\mu_1$.
Although seemingly restrictive, this process is widespread in computational statistics and generative modelling, e.g., optimal-transport displacement interpolation and modern flow-based generative modeling where linear sample-to-sample paths or linear noise–data blends define the training trajectory; see \cite{mccann1997convexity,lipman2022flow,liu2022marginalpreserving,albergo2023stochastic}.

The main result of this section is that the flow generated by the ensemble velocity of an affine process can be of straight-line type if and only if the coupling $\gamma$ between the endpoints is deterministic, i.e., iff there is some measurable map $T: \Rd \to \Rd$ such that $T(X) = Y$ almost surely, or equivalently $\gamma = (\textup{id} \times T)_\sharp \mu_0$, where $\textup{id} \times T: D \to \Rd \times \Rd$ acts as $x \mapsto (x,T(x))$.

The key input, here, is that an affine process has vanishing ensemble acceleration $a_t (x) = \mathbb{E} [ \ddot X_t \mid X_t = x ] \equiv 0$ since, of course, it has no acceleration at the ``particle'' level, $\ddot X_t \equiv 0$.
Thus, our straight line flow characterization given by equation~\eqref{eq:reformulation-2} becomes
\begin{equation}
    \label{eq:vanishing-divergence}
    \nabla \cdot \left( \rho_t \, \Pi_t \right) = 0 \, .
\end{equation}
\begin{theorem}
    \label{thm:affine_processes}
  Under the additional assumptions that
  $\mathbb{E} \, \| X_t \|^2  < \infty$ and $\mathbb{E} \, \| \Pi_t(X_t) \|^2 < \infty$
  for all $t \in[0,1]$ the equation 
  $
    \nabla \cdot \left( \rho_t \, \Pi_t \right) = 0
  $
  implies that there exists a measurable map $T: D \to D$ satisfying $T_\sharp \mu_0 = \mu_1$ such that
  $
    \left( \textup{id} \times T \right)_\sharp \mu_0 = \gamma \, ,
  $
  where $\textup{id}: \Rd \to \Rd$ is the identity map.
  Moreover, if there exists such a continuously differentiable map $T$ with Jacobian $\nabla T$ having no zero or negative singular values, then equation~\eqref{eq:vanishing-divergence} holds.
\end{theorem}
Please find the proof in Appendix~\ref{appendix-A}. This result characterizes all affine processes inducing straight flows as those with deterministically coupled endpoints.

\subsection{Geometric constraints on arbitrary processes.}
In this section, we fix some process $(X_t)_{t \in [0,1]}$ satisfying~\eqref{eq:reformulation-2} and derive some necessary conditions constraining the geometry of the sample paths of $X$.
Namely, by repeating the argument of Theorem~\ref{thm:affine_processes} we obtain:
\begin{theorem}
    \label{thm:geometric-constraints}
    Under the assumptions
        $\mathbb{E} \, \| X_t \|^2  < \infty $, 
        $\mathbb{E} \, \| \Pi_t(X_t) \|^2 < \infty$, and
        $\mathbb{E} \, \| a_t(X_t) \|^2 < \infty$,
    any process $X$ satisfying~\eqref{eq:reformulation-2} satisfies
    $
        \label{eq:geometric-constraints}
        -\mathbb{E}\Big[ \, \textup{Tr} \, \Pi_t(X_t) \, \Big] = \mathbb{E} \left[X_t \cdot \ddot X_t \right]
    $
\end{theorem}
This result links the expected (rescaled) \emph{radial acceleration} with the integral of the trace of the covariance tensor.
Since this trace, however, consists of non-negative quantities, and combining with the identity $\partial_{tt} \, \| X_t \|^2 = 2 \, X_t \cdot \ddot X_t + 2 \, \| \dot X_t \|^2$ we obtain:
\begin{corollary}
    Any process satisfying the integrability assumptions of Theorem~\ref{thm:geometric-constraints} together with equation~\eqref{eq:reformulation-2} satisfies the identities:
    \begin{align*}
        & (1) \quad \mathbb{E} \left[ X_t \cdot \ddot X_t \right] \leq 0 \, , \\
        & (2) \quad \mathbb{E} \Big[ \partial_{tt} \, \| X_t \|^2 \Big] \leq 2 \, \mathbb{E}\left[ \| \dot X_t \|^2 \right] \, , \\
        % & (3) \quad \mathbb{E} \left[ \partial_{tt} \, \| X_t \|^2 \right] \geq 2 \, \sum_{i} \, \mathbb{E} \, \textup{Var}(\dot X_t^i \, | \, X_t) \, .
    \end{align*}
\end{corollary}
Though perhaps opaque at first glance, we believe these identities can provide meaningful insight into the structure of non-trivial solutions to the PDE~\eqref{eq:reformulation-2}. 

\section{Discussion}
We have developed a novel PDE characterization of straight-line flows generated by the ensemble velocity of a stochastic process indexed on the unit time interval, with given marginals $\mu_0$ and $\mu_1$. Our main insight is a characterization of straight line flows in terms of a balance law that links the conditional variance tensor, i.e., the Reynolds tensor, to the ensemble acceleration field.
Using this characterization, we (i) showed that affine processes yielding straight-line flows must have deterministically coupled endpoints; and (ii) derived necessary conditions constraining the geometry of arbitrary processes that induce straight flows.

\paragraph{Limitations.}
Our analysis is {fully theoretical} and operates under regularity assumptions, both on the marginals of the process $X$, e.g., absolute continuity and positivity of the density, as well on the sample paths of $X$, e.g., that they are at least in $W^{2,2}$, to make sense of the velocity and acceleration fields. While applicable to many modern flow-based models, this rules out other important processes, such as diffusions. 
\paragraph{Open directions.}
\begin{itemize}
    \item \emph{Non-trivial solutions to \eqref{eq:reformulation-2}.} Construct processes with non-trivial acceleration fields $a_t \not\equiv 0$ that satisfy~\eqref{eq:reformulation-2}. 
    Such processes might give rise to novel sampling dynamics, given the empirical effectiveness of other processes inducing straight line flows.
    \item \emph{Full characterization of straightness.} Derive necessary and sufficient conditions on the process $X$ under which the PDE~\eqref{eq:reformulation-2} holds. Of particular interest is the case of $X_t = F(t, X_0, X_1)$ for a sufficiently regular $F: [0,1] \times \Rd \times \Rd \to \Rd$ and a random vector $(X_0, X_1) \sim \mu_0 \otimes \mu_1$. This is precisely the setting of \emph{stochastic interpolants}~\cite{albergo2022building,albergo2023stochastic} and readily leads to algorithmic insights.
    \item \emph{Process classes.} Extend the characterization to SDEs (drift–diffusion pairs), non-Markovian processes, or manifold-constrained processes.
\end{itemize}

\paragraph{Outlook.}
We view the balance law \(\nabla \!\cdot\! (\rho_t \Pi_t) = \rho_t a_t\) and its consequences as a compact organizing principle for \emph{geometry-aware} transport design. While we extract only a few implications here, we expect this perspective to guide principled constructions of stochastic processes for \emph{sampling}, and to catalyze empirical investigations into when---and how---straightness can be achieved in practice.

\begin{ack}
PT and YM acknowledge support from the US Air Force Office of Scientific Research (AFOSR) MURI program, under award number FA9550-20-1-0397.  
\end{ack}

\newpage
\medskip
{
    \small
    \bibliography{NEURIPS25-references}
    \bibliographystyle{alpha}
}

%%%%%%%%%%%%%%%%%%%%%%%%%%%%%%%%%%%%%%%%%%%%%%%%%%%%%%%%%%%%

\appendix

\section{Proofs}\label{appendix-A}
\input{appendix-A.tex}

\end{document}

%% file: appendix-A.tex
\subsection{Notation}
\label{subsec:notation}
Let $(\Omega, \mathcal{F}, \mathbb{P})$ be an abstract probability space.
Random variables are measurable functions $X: \Omega \to \Rd$,~\footnote{In this paper, subsets of $\Rd$ are always equipped with the Borel $\sigma$-algebra.} and the expectation $\mathbb{E}[X]$ is defined as the integral $\int_\Omega X \, d\mathbb{P}$. 
For random variables $X, Y: \Omega \to \Rd$ the conditional expectation $\mathbb{E}\left[ X \mid Y \right]$ is defined at the conditional expectation 
$\mathbb{E} \left[ X \mid \sigma(Y) \right]$ where $\sigma(Y)$ is the $\sigma$-algebra generated by $Y$ and for $y \in \R$ the conditional expectation
$\mathbb{E} \left[ X \mid Y=y \right]$ is the integral $\int_\Omega X \, d\mathbb{P}_y$ and $\mathbb{P}_y$ is the \emph{disintegration}~\cite[Section 10.6]{bogachev2007measure} of $\mathbb{P}$ on the level sets of $Y$. 
Finally, given another measurable space $(\Omega', \mathcal{F}')$ we denote the set of measures on $\Omega'$ by $\mathcal{P}(\Omega')$ and the subset of absolutely continuous measures by $\mathcal{P}_{\textup{a.c.}}(\Omega')$.

Let $C^k([0,1]; \Rd)$ be the space of component-wise $k$-times continuously differentiable functions from $[0,1]$ to $\Rd$ and $W^{k,p}([0,1]\to\Rd)$ to be the space of component-wise $k$-times weakly differentiable functions $[0,1] \to \Rd$ with weak derivatives in $L^p([0,1])$.
An $\Rd$-valued stochastic process is a collection $\{X_t\}_{t \in I}$ indexed by some set $I$, where $X: \Omega \to \Rd$ is measurable.
In this paper, we take $I = [0,1]$ and write $X \coloneq (X_t)_t \coloneq (X_t)_{t \in [0,1]}$. 
A sample path of a stochastic process is the function $t \mapsto X_t(\omega)$ for a fixed realization $\omega \in \Omega$.
We note that a stochastic process with sample paths in $W^{k,p} \coloneq W^{k,p}([0,1]; \Rd)$ can equivalently be viewed as a random variable $X: \Omega \to W^{k,p}$ and its 
law is thus in $\mathcal{P}(W^{k,p})$.
\subsection{Proofs}
%%%%%%%%%%%%%%
% Theorem 1
%%%%%%%%%%%%%
\begin{proposition}
   The following are equivalent:
   \begin{enumerate}
        \item For all $(t, x) \in [0,1] \times D$, one has
        \begin{equation*}
            \frac{d^2}{dt^2} \phi_t(x) = 0 \, .
        \end{equation*}
        \item For all $(t, x) \in [0,1] \times D$, one has
        \begin{equation*}
            \phi_t(x) = (1-t) \, x + t \, \phi_1(x) \, ,
        \end{equation*}
        and $(\phi_1)_\sharp \, \mu_0 = \mu_1$.
        \item For all $(t, x) \in [0,1] \times D$, one has
        \begin{equation*}
            D_t v_t(x) = 0 \, .
        \end{equation*}
   \end{enumerate}
\end{proposition}
\begin{proof}
    It is clear that $(2) \implies (1)$
    Let us show that $(1) \implies (2)$.
    We have
    \begin{equation*}
        \frac{d^2}{dt^2} \phi_t(x) = 0 \quad \implies \quad \exists c \in \Rd \, : \, \partial_t \phi_t(x) = c \, \, ,
    \end{equation*}
    so using the boundary condition $\phi_t(x) = x$ we get
    \begin{equation*}
        \phi_t(x) = x + t \, c \, .
    \end{equation*}
    This further implies
    \begin{equation*}
        c = \phi_1(x) - x  \, ,
    \end{equation*}
    and doing some algebra we get
    \begin{equation*}
        \phi_t(x) = (1-t) \, x + t \, \phi_1(x) \, .
    \end{equation*}
    The fact that $(\phi_1)_\sharp \, \mu_0 = \mu_1$ follows from the boundary condition $\phi_1(X_0) \sim \mu_1$.
    
    Finally, let us show that $(1) \iff (3)$. By definition we have
    \begin{equation*}
        \partial_t \phi_t(x) = v_t(\phi_t(x)) \, ,
    \end{equation*}
    and, therefore,
    \begin{align*}
        \frac{d^2}{dt^2} \phi_t(x) &= \partial_t v_t(\phi_t(x)) \\
        &= \partial_t v_t(\phi_t(x)) + \left( v_t(x) \cdot \nabla \right) v_t(\phi_t(x)) \\
        &= D_{v_t(\phi_t(x))} v_t(\phi_t(x)) \, .
    \end{align*}
    Here, we note that by Assumption~\ref{assumptoion:regularity} the map $\phi_t: D \to D$ is surjective.
    Indeed, since the domain of $\phi_t$ is compact and $\phi_t$ is continuous, the image $\phi_t(D)$ is compact and hence closed.
    Now assume there is $x \in D^\circ \cap \phi(D)^c$, where $D^\circ$ is the topological interior of $\phi_t$.
    As a finite intersection of open sets, this set is open; hence there is an open neighborhood $x \in U \subset D^\circ$ with $U \cap \phi(D)^c = \emptyset$.
    But this means that $U \cap \textup{supp} \rho_t = \emptyset$, since $\rho_t = (\phi_t)_\sharp \, \rho_0$, contradicting the assumed positivity of $\rho_t$ in $D$. 
    Thus, we have shown that $\phi_t: D \to D^\circ$ is surjective and since $\phi_t$ is continuous we can extend it uniquely to $\partial D$, hence, w.l.o.g. we have that $\phi_t: D \to D$ is surjective.
    Thus, we can conclude that
    \begin{equation*}
        D_t v_t(x) = 0 \quad \iff \quad \frac{d^2}{dt^2} \phi_t(x) = 0 \, .
    \end{equation*}
\end{proof}
%%%%%%%%%%%%%%
% Theorem 2
%%%%%%%%%%%%%
\begin{lemma}
  The following equation holds for all $t \in [0,1]$ and $x \in D \subseteq \Rd$:
  \begin{equation*}
    \partial_t(\rho_t \, v_t) + \nabla \cdot (\rho_t \, \Sigma_t) = \rho_t \, a_t \, .
  \end{equation*}
\end{lemma}
\begin{proof}
  For a vector valued test function $\Phi \in C^\infty_c(D; \Rd)$, we have
  \begin{align*}
    \int  \Phi(x) \cdot \partial_t(\rho_t(x) \, v_t(x)) \, dx &= \partial_t \int \Phi(x) \cdot (\rho_t(x) \, v_t(x)) \, dx \\
    &= \partial_t \int \Phi(x) \cdot v_t(x) \, d\rho_t(x) \\
    &= \partial_t \mathbb{E} \left[ \, \Phi(X_t) \cdot \mathbb{E} \left[ \dot X_t \mid X_t \right] \, \right] \\
    &= \partial_t \mathbb{E} \left[ \, \Phi(X_t) \cdot \dot X_t \, \right] \\
    &= \mathbb{E} \left[ \, \Phi(X_t) \cdot \ddot X_t \, \right] + \mathbb{E} \left[ \, \left( \nabla \Phi(X_t) \, \dot X_t \right) \cdot \dot X_t \, \right] \\
    &= \mathbb{E} \left[ \, \Phi(X_t) \cdot a_t(X_t) \, \right] + \mathbb{E} \left[ \, \nabla \Phi(X_t) : \dot X_t \otimes \dot X_t \, \right] \\
    &= \mathbb{E} \left[ \, \Phi(X_t) \cdot a_t(X_t) \, \right] + \mathbb{E} \left[ \, \nabla \Phi(X_t) : \Sigma_t(X_t) \right] \\
    &= \int \Phi(x) \cdot \left(\rho_t(x) \, a_t(x)\right) \, dx + \int \nabla \Phi(x) : \left( \rho_t(x) \, \Sigma_t(x) \right) \, dx \\
    &= \int \Phi(x) \cdot \left(\rho_t(x) \, a_t(x)\right) \, dx - \int \Phi(x) \cdot \left[ \nabla \cdot \left( \rho_t(x) \, \Sigma_t(x) \right) \right] \, dx \\
  \end{align*}
  where we have only used definitions and the integration by parts formula.
  Since the above holds for all $\Phi \in C^\infty_c(D; \Rd)$, we can conclude that
  \begin{equation*}
    \partial_t(\rho_t \, v_t) + \nabla \cdot (\rho_t \, \Sigma_t) = \rho_t \, a_t \, ,
  \end{equation*}
  concluding the proof.
\end{proof}
%%%%%%%%%%%%%%
% Theorem 3
%%%%%%%%%%%%%
\begin{lemma}
  The following equation holds for all $(x,t) \in [0,1] \times D$:
  \begin{equation*}
    \rho_t \, D_t v_t + \nabla \cdot \left( \rho_t \, \Pi_t \right) = \rho_t \, a_t \, .
  \end{equation*}
\end{lemma}
\begin{proof}
    By the definition of the Reynolds stress tensor we have
    \begin{equation*}
      \Sigma_t(x) = v_t \otimes v_t + \Pi_t(X_t) \, .
    \end{equation*}
  Now write $\pi^{(i)}_t \in \Rd$ for the $i$-th row of $\Pi_t$ and $v^i_t, a^i_t \in \R \times \R$ for the $i$-th components of $v_t$ and $a_t$, respectively.
  Plugging the above display into~\eqref{eq:second-order-continuity} and looking at the $i$-th component of the resulting vector we obtain
  \begin{equation*}
    \partial_t\left( \rho_t \, v_t^i \right) + \nabla \cdot \left( \rho_t \, v^i_t \, v_t  \right) + \nabla \cdot \left( \rho_t \, \pi^{(i)}_t \right)  = \rho_t \, a^i_t \, .
  \end{equation*}
  Now expanding the differential operators we have
  \begin{equation*}
    v^i_t \, \partial_t \rho_t + \rho_t \, \partial_t v^i_t + v^i_t \, \nabla \cdot (\rho_t \, v_t) + \rho_t \, v_t \cdot \nabla v^i_t + \nabla \cdot (\rho_t \, \pi^{(i)}_t) = \rho_t \, a^i_t \, .
  \end{equation*}
  and rearranging
  \begin{equation*}
    v^i_t \, \Big( \partial_t \rho_t + \nabla \cdot (\rho_t \, v_t) \Big) + \rho_t \, \Big( \partial_t v^i_t + v_t \cdot \nabla v^i_t \Big) + \nabla \cdot (\rho_t \, \pi^{(i)}_t) = \rho_t \, a^i_t \, .
  \end{equation*}
  Using the continuity equation~\eqref{eq:continuity_equation} the first parenthesis vanishes and vectorizing the equation we obtain
  \begin{equation*}
    \rho_t \, \left( \partial_t v_t + v_t \cdot \nabla v_t \right) + \nabla \cdot (\rho_t \, \Pi_t) = \rho_t \, a_t \, .
  \end{equation*}
  Finally, using the definition of the material derivative we conclude.
\end{proof}
\begin{corollary}
  \label{corollary:straightness}
  For all $(t,x) \in [0,1] \times D$ we have:
  \begin{equation*}
    D_t v_t = 0 \, \iff \, \nabla \cdot \left( \rho_t \, \Pi_t \right) = \rho_t \, a_t \, .
  \end{equation*}
\end{corollary}
%%%%%%%%%%%%%%
% Theorem 5
%%%%%%%%%%%%%%
\begin{theorem}
  Under the additional assumptions that
  \begin{align*}
    &\mathbb{E} \, \| X_t \|^2  < \infty \, \textup{and} \, \\
    &\mathbb{E} \, \| \Pi_t(X_t) \|^2 < \infty \, ,
  \end{align*}
  for all $t \in[0,1]$ the equation
  \begin{equation*}
    \nabla \cdot \left( \rho_t \, \Pi_t \right) = 0 \, ,
  \end{equation*}
  implies that there exists a measurable map $T: \Rd \to \Rd$ satisfying $T_\sharp \mu_0 = \mu_1$ such that
  \begin{equation*}
    \left( \textup{id} \times T \right)_\sharp \mu_0 = \gamma \, ,
  \end{equation*}
  where $\textup{id}: \Rd \to \Rd$ is the identity map.
  Moreover, if there exists such a continuously differentiable map $T$ with Jacobian $\nabla T$ having no zero or negative singular values, then it holds that
  \begin{equation*}
    \nabla \cdot \left( \rho_t \, \Pi_t \right) = 0 \, .
  \end{equation*}
\end{theorem}
\begin{proof}
  For $D^\circ \subset \Rd$, fix $R > 0$ and take a cut-off function $\eta_R \in C^\infty_c(D)$ such that $\supp(\eta_R) \subset B_{2R}(0)$ and $\supp(1 - \eta_R) \subset B_R(0)^c$.
  Now consider the test function $\Phi_R(x) \in C^\infty_c(D; \Rd)$ given by $\Phi_R(x) = \eta_R(x) \, x$ and compute
  \begin{align*}
    &\int \Phi_R(x) \cdot \left( \nabla \cdot \left(\rho_t \, \Pi_t \right) \right) \, dx =  \\
    &-\int \nabla \Phi_R(x) : \left( \rho_t \, \Pi_t \right) \, dx = \\
    &- \int \nabla \Phi_R(x) : \Pi_t(x) \, d\rho_t(x) = \\
    &- \int_{B_R(0)} I_d : \Pi_t(x) \, d\rho_t(x) + \int_{B_R(0)^c} \nabla \Phi_R(x) : \Pi_t(x) \, d\rho_t(x) = \\
    &- \int_{B_R(0)} \textup{Tr} \, \Pi_t(x) \, d\rho_t(x) + \int_{B_R(0)^c} \nabla \Phi_R(x) : \Pi_t(x) \, d\rho_t(x) = \\
  \end{align*}
  Taking the limit $R \to \infty$ and using Lemma~\ref{lemma:integrability} together with dominated convergence we obtain
  \begin{equation*}
    \int_{\Rd} \textup{Tr} \, \Pi_t(x) \, d\rho_t(x) = 0
  \end{equation*}
  which is equivalent to
  \begin{equation*}
    \sum_{i=1}^d \int_{\Rd} \textup{Var}\left( \, \dot X_t^i \, | \, X_t = x \, \right) \, d\rho_t(x) = 0 \, .
  \end{equation*}
  Now since $\textup{Var}(\dot X_t^i | X_t = x) \geq 0$ we have that
  \begin{equation*}
    \textup{Var}\left( \, \dot X_t^i \, | \, X_t = x \, \right) = 0 \text{ for $\rho_t$-almost every } x \in \Rd.
  \end{equation*}
  Thus, there is a Borel measurable function $G: \Rd \to \Rd$ such that for all $t \in [0,1]$
  \begin{equation*}
    \dot X_t = G(X_t) \text{ almost surely. }
  \end{equation*}
  Taking $t = 0$ this reads $Y = X + G(X)$ almost surely, thus after setting $T = \textup{id} + G$ we have
  \begin{equation*}
    Y = T(X) \text{ almost surely. }
  \end{equation*}
  This concludes the first part of the proof.
  For the second part, we follow the proof of \cite[Theorem 3.4]{marzouk2025approximation}.
  Using the assumptions on the map $T$ we have that the map
  \begin{equation*}
    G: (x,t) \mapsto \big( (1-t) \, x + t \, T(x), t \big) \, ,
  \end{equation*}
  is a bijection onto its image.
  Therefore, there is an inverse map $S: \textup{im} \,  G \to \Rd$ such that
  \begin{equation*}
    S\big( (1-t) \, x + t \, T(x), t \big) = x \, ,
  \end{equation*}
  which allows us to conclude that for each $t \in [0,1]$ we have
  \begin{equation*}
    S(X_t, t) = X \text{ a.s. }
  \end{equation*}
  Finally, recall that 
  \begin{equation*}
    \dot X_t = T(X) - X \text{ a.s. }
  \end{equation*}
  and thus
  \begin{equation*}
    \dot X_t = T\big( S(X_t, t) \big) - S(X_t, t) \text{ a.s. }
  \end{equation*}
  showing that at each $t \in [0,1]$ the random variable $\dot X_t$ is a measurable function of $X_t$.
  This implies that $\textup{Var}(\dot X_t | X_t) \equiv 0$ almost surely, completing the proof.
\end{proof}
%%%%%%%%%%%%%%
\begin{lemma}
  \label{lemma:integrability}
  Assume that 
  \begin{align*}
    &\mathbb{E} \, \| X_t \|^2  < \infty \, \textup{and} \, \\
    &\mathbb{E} \, \| \Pi_t(X_t) \|^2 < \infty \, ,
  \end{align*}
  for all $t \in [0,1]$.
  Then, for all $t, x \in [0,1] \times \Rd$ there is a collection of random variables $Y_t: \Omega \to \R$ such that
  \begin{equation*}
    Y_t \in L^1 \, ,
  \end{equation*}
  and
  \begin{equation*}
    \Big| \nabla \Phi_R(X_t) : \Pi_t(X_t) \Big| \leq Y_t \, ,
  \end{equation*}
  uniformly in $R > 0$.
\end{lemma}
\begin{proof}
  Start by noticing that the cut-off function $\eta_R$ can be chosen such that
  \begin{equation*}
    \| \, \eta_R \, \|_{W^{1,\infty}} \leq C
  \end{equation*}
  for some constant $C > 0$ independent of $R$ and so for all $x\in \Rd$
  \begin{equation*}
    \left| \nabla \Phi_R(x) \right| \lesssim \| x \| \, .
  \end{equation*}
  Thus, for any $x \in \Rd$ we can write
\begin{align*}
  \Big| \nabla \Phi_R(x) : \Pi_t(x) \Big| &\leq \sum_{ij} \Big| \nabla \Phi_R(x) \Big| \, \left| \, \Pi_t^{ij}(x) \, \right| \\
  & \lesssim \| x \| \, \sum_{ij} \left| \, \Pi_t^{ij}(x) \, \right| \\
  & \lesssim \| x \|^2 + \Big( \sum_{ij} \left| \, \Pi_t^{ij}(x) \, \right| \Big)^2 \\
  & \lesssim \| x \|^2 + \sum_{ij} \left| \, \Pi_t^{ij}(x) \, \right|^2
\end{align*}
where in the third line we used the Cauchy-Schwartz inequality and in the third line we used Jensen's inequality, suppressing constants depending on the dimension $d$.
Now we can set
\begin{equation*}
  Y_t = \| X_t \|^2 + \sum_{ij} \left| \, \Pi_t^{ij}(X_t) \, \right|^2
\end{equation*}
which is clearly in $L^1$ by the assumptions in the statement of the lemma.
\end{proof}
\begin{theorem}
    Under the assumptions
    \begin{align*}
        &\mathbb{E} \, \| X_t \|^2  < \infty \, \textup{and} \, , \\
        &\mathbb{E} \, \| \Pi_t(X_t) \|^2 < \infty \, , \\
        &\mathbb{E} \, \| a_t(X_t) \|^2 < \infty \, .
    \end{align*}
    any process $X$ satisfying~\eqref{eq:reformulation-2} satisfies
    \begin{equation}
        \label{eq:geometric-constraints}
        -\mathbb{E}\Big[ \, \textup{Tr} \, \Pi_t(X_t) \, \Big] = \mathbb{E} \left[X_t \cdot \ddot X_t \right]
    \end{equation}
\end{theorem}
\begin{proof}
  The proof follows exactly the same steps as the proof of Theorem~\ref{thm:affine_processes} with the additional application of the Lebesgue dominated convergence theorem in the quantity
  \begin{equation*}
    \mathbb{E} \Big[ \eta_R(X_t) \cdot a(X_t) \Big] = \mathbb{E} \Big[ \eta_R(X_t) \cdot \ddot X_t \Big] \, ,
  \end{equation*}
  which is justified since by the Cauchy-Schwartz inequality we have
  \begin{equation*}
    \left| \mathbb{E} \Big[ \eta_R(X_t) \cdot a(X_t) \Big] \right|^2 \leq \mathbb{E} \Big[ \| a(X_t) \|^2 \Big] \, \mathbb{E} \Big[ \| X_t \|^2 \Big] < \infty \, ,
  \end{equation*}
  using a Cauchy-Schwartz inequality and $\left| \eta_R(X_t) \right| \leq |X_t|$ by the definition of $\eta_R$.
\end{proof}
%%%%%%%%%%%%%%%%%
% Cont. equation
%%%%%%%%%%%%%%%%%
\begin{lemma}[Continuity Equation]
  \label{lemma:continuity_equation}
  We have the identity
  \begin{equation*}\label{eq:continuity_equation}
    \partial_t \,  \rho_t + \nabla \cdot \left( \rho_t \, v_t \right) = 0
  \end{equation*}
  for all $t \in [0,1]$.
\end{lemma}
\begin{proof}
  For $D^\circ \subset \Rd$, fix a test function $\varphi \in C^\infty_c(D)$ and compute
  \begin{align*}
    \int_{\Rd} \varphi(x) \, \partial_t \, \rho_t(x) \, dx 
    &= \partial_t \, \int_{\Rd} \varphi(x) \, d \rho_t(x) \\
    &= \partial_t \, \mathbb{E} \left[ \varphi(X_t) \right] \\
    &= \mathbb{E} \left[ \nabla \varphi(X_t) \cdot \dot X_t \right] \\
    &= \mathbb{E} \Big[ \nabla \varphi(X_t) \cdot v_t(X_t) \Big] \\
    &= \int_{\Rd} \nabla \varphi(x) \cdot v_t(x) \, d\rho_t(x) \\
    &= -\int_{\Rd} \varphi(x) \, \nabla \cdot \left( \rho_t(x) \, v_t(x) \right) \, dx
  \end{align*}
  where we used integration by parts, the fact that $\varphi$ has compact support as well as properties of the conditional expectation.
  Since the above holds for all $\varphi \in C^\infty_c(\Rd)$ we get the desired result.
\end{proof}

%% file: NEURIPS25-references.bib
@article{albergo2023stochastic,
  title={Stochastic interpolants: A unifying framework for flows and diffusions},
  author={Albergo, Michael S and Boffi, Nicholas M and Vanden-Eijnden, Eric},
  journal={arXiv preprint arXiv:2303.08797},
  year={2023}
}

@article{liu2022flow,
  title={Flow straight and fast: Learning to generate and transfer data with rectified flow},
  author={Liu, Xingchao and Gong, Chengyue and Liu, Qiang},
  journal={arXiv preprint arXiv:2209.03003},
  year={2022}
}

@article{marzouk2025approximation,
  title={Distribution learning via neural differential equations: minimal energy regularization and approximation theory},
  author={Marzouk, Youssef and Ren, Zhi and Zech, Jakob},
  journal={arXiv preprint arXiv:2502.03795},
  year={2025}
}

@article{albergo2022building,
  title={Building normalizing flows with stochastic interpolants},
  author={Albergo, Michael S and Vanden-Eijnden, Eric},
  journal={arXiv preprint arXiv:2209.15571},
  year={2022}
}

@misc{lipman2022flow,
      title={Flow Matching for Generative Modeling}, 
      author={Yaron Lipman and Ricky T. Q. Chen and Heli Ben-Hamu and Maximilian Nickel and Matt Le},
      year={2023},
      eprint={2210.02747},
      archivePrefix={arXiv},
      primaryClass={cs.LG},
      url={https://arxiv.org/abs/2210.02747}, 
}

@misc{tong2023cfm,
      title={Improving and generalizing flow-based generative models with minibatch optimal transport}, 
      author={Alexander Tong and Kilian Fatras and Nikolay Malkin and Guillaume Huguet and Yanlei Zhang and Jarrid Rector-Brooks and Guy Wolf and Yoshua Bengio},
      year={2024},
      eprint={2302.00482},
      archivePrefix={arXiv},
      primaryClass={cs.LG},
      url={https://arxiv.org/abs/2302.00482}, 
}

@misc{song2021scorebased,
      title={Score-Based Generative Modeling through Stochastic Differential Equations}, 
      author={Yang Song and Jascha Sohl-Dickstein and Diederik P. Kingma and Abhishek Kumar and Stefano Ermon and Ben Poole},
      year={2021},
      eprint={2011.13456},
      archivePrefix={arXiv},
      primaryClass={cs.LG},
      url={https://arxiv.org/abs/2011.13456}, 
}

@misc{liu2022marginalpreserving,
      title={Rectified Flow: A Marginal Preserving Approach to Optimal Transport}, 
      author={Qiang Liu},
      year={2022},
      eprint={2209.14577},
      archivePrefix={arXiv},
      primaryClass={stat.ML},
      url={https://arxiv.org/abs/2209.14577}, 
}

@misc{bansal2025wasserstein,
      title={On the Wasserstein Convergence and Straightness of Rectified Flow}, 
      author={Vansh Bansal and Saptarshi Roy and Purnamrita Sarkar and Alessandro Rinaldo},
      year={2025},
      eprint={2410.14949},
      archivePrefix={arXiv},
      primaryClass={cs.LG},
      url={https://arxiv.org/abs/2410.14949}, 
}

@book{bogachev2007measure,
  title={Measure theory},
  author={Bogachev, Vladimir I},
  year={2007},
  publisher={Springer}
}

@article{mccann1997convexity,
  author  = {Robert J. McCann},
  title   = {A Convexity Principle for Interacting Gases},
  journal = {Advances in Mathematics},
  volume  = {128},
  number  = {1},
  pages   = {153--179},
  year    = {1997},
  doi     = {10.1006/aima.1997.1634}
}

@article{hertrich2025relation,
  title={On the Relation between Rectified Flows and Optimal Transport},
  author={Hertrich, Johannes and Chambolle, Antonin and Delon, Julie},
  journal={arXiv preprint arXiv:2505.19712},
  year={2025}
}
